\documentclass[11pt]{article} 

\usepackage{etoolbox}

\newtoggle{arxiv}
\toggletrue{arxiv}

\iftoggle{arxiv}
{
\usepackage{authblk} 
\usepackage[dvipsnames]{xcolor}
\definecolor{BrickRed}{rgb}{0.6,0,0} 
\usepackage{amsmath, amssymb, graphicx, url}
\usepackage[ruled]{algorithm2e}
\usepackage[square,numbers]{natbib}
\usepackage{fullpage}
\usepackage[small,bf]{caption}
\setlength{\captionmargin}{24pt}

\usepackage[colorlinks=true,
linkcolor=RedOrange,urlcolor=RedOrange,citecolor=MidnightBlue]{hyperref}

\newtheorem{theorem}{Theorem}[section]
\newtheorem{lemma}[theorem]{Lemma}
\newtheorem{proposition}[theorem]{Proposition}
\newtheorem{corollary}[theorem]{Corollary}

\numberwithin{equation}{section}

\newcommand{\qed}{\rule{7pt}{7pt}}
\newenvironment{proof}{\noindent{\textbf{Proof}}\hspace*{.5em}}{\qed\vspace{0.125in}}

\renewcommand{\cite}{\citet}
}
{
}

\newtheorem{assumption}[theorem]{Assumption}

\newcommand{\R}{\mathbb{R}}
\newcommand{\E}{\mathbb{E}}
\renewcommand{\Pr}{\mathbb{P}}

\newcommand{\minimize}{\text{minimize}}
\newcommand{\st}{\text{subject to}}

\newcommand{\spn}{\operatorname{span}}

\renewcommand{\phi}{\varphi}
\renewcommand{\theta}{\vartheta}

\def\ddefloop#1{\ifx\ddefloop#1\else\ddef{#1}\expandafter\ddefloop\fi}
\def\ddef#1{\expandafter\def\csname c#1\endcsname{\ensuremath{\mathcal{#1}}}}
\ddefloop ABCDEFGHIJKLMNOPQRSTUVWXYZ\ddefloop
\def\ddef#1{\expandafter\def\csname s#1\endcsname{\ensuremath{\mathsf{#1}}}}
\ddefloop ABCDEFGHIJKLMNOPQRSTUVWXYZ\ddefloop
\def\ddef#1{\expandafter\def\csname b#1\endcsname{\ensuremath{\mathbf{#1}}}}
\ddefloop ABCDEFGHIJKLMNOPQRSTUVWXYZ\ddefloop
\def\ddef#1{\expandafter\def\csname b#1\endcsname{\ensuremath{\mathbf{#1}}}}
\ddefloop abcdefghijklmnopqrstuvwxyz\ddefloop

\def\E{\mathop{\mathbb{E}}}
\def\Pr{\mathop{\mathbb{P}}}
\def\var{\operatorname{Var}}
\def\Reals{\mathbb{R}}

\def\dif#1{\mathrm{d}{#1}\,} 

\usepackage{xcolor}
\definecolor{BrickRed}{rgb}{0.6,0,0}
\usepackage{marginnote} 

\usepackage{bbm}
\makeatletter
\def\set@curr@file#1{\def\@curr@file{#1}} 
\makeatother
\usepackage[load-configurations=version-1]{siunitx} 

\iftoggle{arxiv}
{
\title{Interpolating Classifiers Make Few Mistakes}

\author[1]{Tengyuan Liang\thanks{\tt Email:\href{tengyuan.liang@chicagobooth.edu}{tengyuan.liang@chicagobooth.edu}. Liang's research is generously supported by the NSF Career Award (DMS-2042473), the George C. Tiao Fellowship and the William S. Fishman Research Fund.}}
\author[2]{Benjamin Recht\thanks{\tt Email:\href{brecht@berkeley.edu}{brecht@berkeley.edu}. Recht's research is generously supported in part by ONR awards N00014-20-1-2497 and N00014-18-1-2833, NSF CPS award 1931853, and the DARPA Assured Autonomy program (FA8750-18-C-0101).}}
\affil[1]{University of Chicago}
\affil[2]{University of California, Berkeley}

\date{}

}
{
\title[Label Interpolation]{Interpolating Classifiers Make Few Mistakes}
\usepackage{times}

\coltauthor{%
 \Name{Tengyuan Liang} \Email{Tengyuan.Liang@chicagobooth.edu}\\
 \addr Booth School of Business\\
University of Chicago 
 \AND
 \Name{Benjamin Recht} \Email{brecht@berkeley.edu}\\
 \addr Department of Electrical Engineering \& Computer Sciences\\
University of California, Berkeley%
}
}

\begin{document}

\maketitle

\iftoggle{arxiv}{\vspace{-0.25in}}{}

\begin{abstract}%
    This paper provides elementary analyses of the regret and generalization of minimum-norm interpolating classifiers (MNIC). The MNIC is the function of smallest Reproducing Kernel Hilbert Space norm that perfectly interpolates a label pattern on a finite data set. We derive a mistake bound for MNIC and a regularized variant that holds for all data sets. This bound follows from elementary properties of matrix inverses. Under the assumption that the data is independently and identically distributed, the mistake bound implies that MNIC generalizes at a rate proportional to the norm of the interpolating solution and inversely proportional to the number of data points. This rate matches similar rates derived for margin classifiers and perceptrons. We derive several plausible generative models where the norm of the interpolating classifier is bounded or grows at a rate sublinear in $n$. We also show that as long as the population class conditional distributions are sufficiently separable in total variation, then MNIC generalizes with a fast rate.
\end{abstract}

\iftoggle{arxiv}
{}
{
}

\section{Introduction}\label{sec:intro}
Using a squared loss for classification problems remains a controversial topic in machine learning. Though popular for regression, the squared loss makes little intuitive or semantic sense for classification problems. A squared loss appears to penalize models for being ``too correct'' when the prediction of a label is much greater than $1$. Moreover, the labels in classification are arbitrary, so the minimum square error doesn't convey much information about the latent prediction problem. Similarly, minimum-norm interpolation, the limiting solution of ridge regression as the regularization parameter tends to zero, appears to be a curious choice for a classifier. Forcing a function to be equal to an arbitrary label set could be unnecessarily aggressive if we only aim to have our classifier predict the correct sign on our data.

Despite these reservations, minimum-norm interpolation classifiers (MNIC) and regularized least-squares classification (RLSC) work exceptionally well in practice. Recent investigations by \cite{Shankar20a} demonstrated that there was no advantage to using more sophisticated classification techniques on popular contemporary data sets. From a theoretical perspective, stochastic gradient descent on a variety of empirical risk minimization problems will converge to the MNIC solution under mild assumptions.  \cite{jacot2018neural} and \cite{heckel2020compressive} have shown that neural nets trained with stochastic gradient descent approximate the minimum-norm solutions of appropriate Reproducing Kernel Hilbert Spaces.

On top of these connections and empirical results, least-squares methods have many attractive features that argue in their favor. Their solutions can be written out algebraically, and we can lean on powerful linear algebra tools to compute them at large scale (see \cite{avron2017faster,ma2017diving,rudi2017falkon,wang2019exact,shankar2018numpywren}, for example).  Many auxiliary quantities such as the leave-one-out error can also be computed in closed form. If these methods are at all competitive, the machine learning community should encourage their use.

In this paper we provide an elementary analysis of MNIC, partially helping to explain the method's success. We first present a mistake bound for minimum-norm interpolation classification that holds unconditionally of how the data was generated. The proof follows immediately from two applications of the matrix inversion lemma. Moreover, we present two simple geometric and algorithmic proofs that shed further insights into the properties of MNIC and RLSC and the structure of QR decompositions. 

Adding the assumption that the data is generated i.i.d., our mistake bound and an application of Markov's inequality implies that MNIC generalizes with a rate of $B_{n}^2/n$ where $B_{n}$ denotes the expected norm of the MNIC when interpolating $n$ i.i.d. sampled data points.  This compares favorably with the generalization bounds of Vapnik and Chervonenkis for the perceptron which show the generalization error scales as $\frac{1}{M_{n+1}^2 n}$ where $M_{n+1}$ denotes the expected size of the margin for a sample of $n+1$ data points \cite{VapnikChervonenkis1974Book}. Since for the maximum margin solution, the margin is the inverse of the norm of the separating hyperplane, our bound tracks a similar scaling as these classic margin-based results.

A key feature of MNIC and RLSC is that the expected norm can be estimated for a variety of probabilistic models of data. By leveraging concentration inequalities and random matrix tools, we can analyze a variety of plausible data generation schemes and show that as long as there is certain separation between the distributions of the two classes, then we can expect the interpolant norm to not grow too quickly.

\section{Related Work}\label{sec:related}
Using least-squares loss for classification has a long history in machine learning and statistics. Its modern popularization was by \cite{suykens1999least} and connections to general kernel machines were studied by \cite{RifRLSC}. While many authors use least-squares classification as their default tool (for example, ~\cite{rudi2017falkon,Shankar20a}), it is still not common practice to use a squared loss for classification. Theory bounds for regression with bounded labels do apply to least-squares classification, and such results exist. For example,~\cite{lee1998importance} showed agnostic learning is possible with a squared loss. The results in this work focus on a particular class of interpolators that are amenable to a simpler analysis.

While we focus on interpolating classifiers, our inspiration comes from work on maximum margin classifiers.  \cite{Novikoff1962} famously derived a mistake bound for the perceptron, and \cite{VapnikChervonenkis1974Book} used this bound and a leave-one-out argument to turn the perceptron mistake bound into a generalization bound. \cite{02.Vapnik.Chapelle} prove a similar generalization bound for the maximum-margin classifier using a similar argument, though their proof requires a strong assumption about support vectors not changing when data is resampled. A more general theory of generalization for margin classifiers using Rademacher complexity was developed by~\cite{koltchinskii2002empirical}. These bounds yielded suboptimal rates in the case when the data was separable, and ~\cite{srebro2010smoothness} provided an argument to yield fast rates. All margin bounds require the size of the margin to not shrink too quickly as the number of data points increases.

A lesser known body of work, initiated by Bernstein, attempted to understand the complexity of estimating a planted linear model using online algorithms. \cite{bernstein1992absolute}, using an online algorithm identical to the one presented in Section~\ref{sec:ridge}, showed that the online error could be bounded by the norm of the planted solution times the maximum norm of the data points. This bound has a similar flavor to our bound, though necessarily assumes a planted solution that is observed via noiseless linear measurements. Later, \cite{klasner1995noise} showed how to extend Bernstein's work to the noisy case by running an online version of ridge regression.  Our work differs from this earlier work in several ways. First, the proof of our main online bound follows from a more direct and short argument. Second, we make no assumptions about the existence of a planted model. Rather, to bound the error with $n$ data points, we merely assume that the interpolation problem is solvable for any data set of size at most $n$. That is, our labels $y$'s are completely arbitrary, and need not be realized by any finite dimensional linear model. Though our work focuses on interpolation, we show how to immediately extend it to ridge regression in Section~\ref{sec:ridge}, yielding improvements on Klasner and Simon's bounds that hold with more generality.

Many recent papers have used generative models to understand how such margins scale (see, for example, \cite{deng2019model,montanari2019generalization,liang2020precise,chatterji2020finite}). In this work, we study scaling on the norm of the interpolating function using similar ideas. Leveraging recent developments in non-linear random matrices, \cite{liang2020JustInterpolate} and later in \cite{liang2020MultipleDescent} studied the generalization of kernel ridgeless regression. \cite{bartlett2020benign} studied the generalization of minimum-norm interpolants in the context of infinite-dimensional Gaussian process regression. \cite{hastie2019surprises} applied precise asymptotic tools to analyze random non-linear feature regression. Though our generalization results can be extended to the case of regression, we focus here on generative models of classification, and show these generalize with mild separation assumptions on the data generating process.

Connections between the Perceptron for SVM and MNIC are recently investigated in \cite{muthukumar2020classification} using a distinctive approach compared to ours. They construct generative models such that with sufficient overparametrization, all data points are support vectors. This means the SVM returns the same solution as MNIC under such models. However, our work directly works with MNIC and does not require a solution coincident with an SVM solution. Moreover, our generalization bounds involve a norm rather than a margin, and we demonstrate that these norms are often simple to calculate or bound for complex models in both theory and practice.

\section{A Mistake Bound for MNIC}\label{sec:mistake}

Let $\mathcal{H}$ be a Reproducing Kernel Hilbert Space (RKHS) with embedding functions $\varphi_x$ so that $k(x,z)=\langle \varphi_x, \varphi_z \rangle_K$. 
Suppose that we are given a data sequence $\{(x_j,y_j)\}_{j\in \mathbb{N}_{+}}$ and consider the MNIC optimization problem on the data set $S_i :=\{(x_j,y_j)\}_{j=1}^i$
\begin{equation}\label{eq:rkhs-interpolate}
	\begin{array}{ll}
	\minimize & \|f\|_K \\
	\st & f(x_j) = y_j,~~j=1,\ldots, i\,.
	\end{array}
\end{equation}
When it exists, let $\widehat{f}_{S_i}$ denote the optimal solution of this optimization problem. 

The key to our analysis will be controlling the error of $\widehat{f}_{S_{i}}$ on the next data point in the sequence, $x_{i+1}$. To do so, we use a simple identity that follows from linear algebra. Let
$$
	s_{i} := \operatorname{dist} \left(\spn( \varphi_{x_1},\ldots, \varphi_{x_{i-1}}),\varphi_{x_{i}}\right)\,,
$$
where the distance measures the RKHS norm of the embedding function.
Then we have the following
\begin{lemma}\label{lemma:sequential-error}
For any $i \in \mathbb{N}_{+}$, and any data sequence $\{ (x_j, y_j)\}_{j \in \mathbb{N}_{+}}$, suppose that the solutions $\widehat f_{S_{i-1}}$ and $\widehat f_{S_{i}}$ in Problem \eqref{eq:rkhs-interpolate} exist. Then they satisfy 
$$
	\big(y_{i}- \widehat{f}_{S_{i-1}}(x_{i}) \big)^2 = s_{i}^2 \big(\| \widehat{f}_{S_{i}}\|_K^2-\| \widehat{f}_{S_{i-1}}\|_K^2 \big)\,.
$$
\end{lemma}

We hold off on proving this Lemma until the next section where we give several proofs that highlight multiple aspects of minimum norm interpolation. With the Lemma in hand, however, we can now state and prove our two immediate results. Given $n$ data points, define in addition
$$
	s_{S_n\backslash i} := \operatorname{dist} \left(\spn( \varphi_{x_1},\ldots, \varphi_{x_{i-1}},\varphi_{x_{i+1}},\ldots,\varphi_{x_{n}}),\varphi_{x_{i}}\right)\,.
$$Our first result is a deterministic mistake bound that makes no assumptions about the data, as long as MNIC is well defined.

\begin{theorem}
	\label{thm:regret}
Let $R_n^2 = \max_{1\leq i\leq n} \| \phi_{x_i} \|_K^2$ 
and $r_n^2 = \min_{1\leq i \leq n} s_{S_n\backslash i}^2$. 
Then we have the following regret bound
\begin{align}\label{eq:regret-bound}
		r_n^2 \cdot \| \widehat f_{S_n} \|_{K}^2 \leq \sum_{i=1}^{n} \big(y_{i} - \widehat{f}_{S_{i-1}} (x_{i}) \big)^2 \leq R_n^2 \cdot \| \widehat f_{S_n} \|_{K}^2 \enspace,
\end{align}
In the binary classification setting with $\cY = \{\pm 1\}$, we further have
	\begin{align}\label{eq:regret-bound01}
		\sum_{i=1}^{n} \mathbbm{1} \big\{y_{i} \widehat{f}_{S_{i-1}} (x_{i}) \leq 0 \big\} \leq R_n^2 \cdot \| \widehat f_{S_n} \|_{K}^2
		\enspace.
	\end{align}
\end{theorem}

\begin{proof}
Using the bounds $r_n^2 \leq s_i^2 \leq R_n^2$ for all $1\leq i\leq n$ and Lemma~\ref{lemma:sequential-error}, we have
$$
r_{n}^2 \big( \| \widehat{f}_{S_{i}}\|_K^2 - \| \widehat{f}_{S_{i-1}}\|_K^2 \big) \leq	\big( y_i - \widehat{f}_{S_{i-1}}(x_i) \big)^2 \leq R_{n}^2 \big( \| \widehat{f}_{S_{i}}\|_K^2 - \| \widehat{f}_{S_{i-1}}\|_K^2 \big)\,.
$$
Adding these inequalities together for $1\leq i\leq n$ proves the inequalities in Equation~\eqref{eq:regret-bound}. Equation~\eqref{eq:regret-bound01} follows because the square loss upper bounds the zero-one loss.
\end{proof}

Note that $R_n^2$ is uniformly bounded over $n$ as long as $\sup_{x \in \cX} K(x, x) < \infty$, since
$R_n^2 = \max_{i} K(x_i, x_i)$. This theorem thus quantifies the difficulty of prediction with a single data-adaptive quantity: $\| \widehat f_{S_n} \|_{K}^2$. In later sections we study precisely how such norm could grow with the sequence length $n$. 

Again, we emphasize that this main result holds for arbitrary data sequence as long as the MNIC is well-defined. This is true, for example, when the empirical kernel matrix has full rank. Such a rank condition holds for any universal kernel, such as the Gaussian kernel.

The above theorem immediately yields the following corollary, a generalization bound for MNIC in the case where the $\{ (x_i, y_i) \}_{i=1}^n$ are sampled $i.i.d.$ from some distribution. 
\begin{theorem}
\label{thm:mnic-generalization}
Suppose that $S_n$ consists of i.i.d. samples from $\mathcal{D}$ and that $y$ denotes a class label in $\{-1,1\}$. Let $(\bx, \by)$ denote a new random draw from the same $\mathcal{D}$. Let $R_n$ be defined as in Theorem~\ref{thm:regret} and let $B_n = \| \widehat f_{S_n} \|_{K}$.
\begin{enumerate}
\item We have
$$
	 \min_{1\leq i \leq n}     \frac{ n  \cdot \Pr[ \by \widehat{f}_{S_{i}}(\bx) < 0]}{\E[R_{n}^2 B_{n}^2]} \leq 1\,.
$$
\item  If either $\E[(1 - \by \widehat{f}_{S_i}(\bx))^2]$ or $\Pr[ \by \widehat{f}_{S_i}(\bx) \leq 0]$ is a non-increasing sequence indexed by $i$, we have
$$
	\Pr[ \by \widehat{f}_{S_n}(\bx) \leq  0] \leq  \frac{\E[R_{n}^2 B_{n}^2 ] }{n}\,.
$$
\item If we denote the Polyak average of $\widehat{f}_{S_i}$ by
$$
	\widetilde{f}_n = \frac{1}{n}\sum_{i=1}^{n} \widehat{f}_{S_i}
$$
then we have
$$
	 \Pr\left[  \by \widetilde{f}_n(\bx) \leq 0 \right] \leq  \frac{\E[R_{n+1}^2 B_{n+1}^2] }{n+1}\,.
$$
\end{enumerate}
\end{theorem}

\begin{proof}
Taking expected values in Equation~\eqref{eq:regret-bound}, we can drop the subscripts on $(x_i, y_i)$ as they are identically distributed to $(\bx, \by) \sim \cD$ and is independent of $S_{i-1}$. This yields the bound
\begin{equation}\label{eq:online}
	\sum_{i=1}^{n} \E[(1 - \by \widehat{f}_{S_{i-1}}(\bx))^2]  = \sum_{i=1}^{n} \E[(1 - y_i \widehat{f}_{S_{i-1}}(x_i))^2] \leq \E[R_{n}^2 B_{n}^2]\,.
\end{equation}
Inequality~\eqref{eq:online} immediately proves
$$
	\min_{0 \leq i\leq n-1} \E[(1 - \by \widehat{f}_{S_i}(\bx))^2] \leq  \frac{\E[R_{n}^2 B_{n}^2]}{n} \,.
$$
Markov's inequality implies that
$$
	\Pr[ \by\widehat{f}_{S_i}(\bx) \leq 0]	\leq \E[(1 - \by \widehat{f}_{S_i}(\bx))^2]\,.
$$
It is clear that $\min_{1\leq i\leq n}\Pr[ \by \widehat{f}_{S_{i}}(\bx) \leq 0] \leq  \min_{1\leq i\leq n-1}\Pr[ \by \widehat{f}_{S_{i}}(\bx) \leq 0]$. Also, the second minimum value stays the same over $0\leq i\leq n-1$ as $\widehat{f}_{S_0} = 0$. Therefore the first part of the theorem is proved.
Similarly, assuming the sequence is non-increasing means that the minimum summand of~\eqref{eq:online} is $\E[(1 - \by \widehat{f}_{n}(\bx))^2]$ (or respectively $\Pr[ \by \widehat{f}_{S_i}(\bx) \leq 0]$). This proves the second part of the theorem. The third part of the theorem follows by applying Jensen's inequality to lower bound the left-hand-side of inequality~\eqref{eq:online} (with $n+1$ substituting $n$), applying Markov's inequality, and then rescaling by and then rescaling by $\frac{n+1}{n}$. 
\end{proof}

Finally, note that we could have derived a result similar to Theorem~\ref{thm:mnic-generalization} for \emph{regression}, replacing probability of error with expected squared loss. For example, the same argument would show
$$
	 \min_{1\leq i \leq n}  \E\left[(\by- \widehat{f}_{S_i}(\bx))^2\right] \leq \frac{\E[R_{n}^2 B_{n}^2]}{n}\,.
$$
However, we note that the norm bound $B_n$ may grow rapidly for such regression problems. Whereas we will see several examples in the sequel where the norm bound $B_n$ grows slowly with $n$ for classification, we leave study of conditions under which these norms grow sublinearly in $n$ for regression to future work.

\section{Proofs of Lemma~\ref{lemma:sequential-error}: Algorithmic and Geometric Perspectives}

Lemma~\ref{lemma:sequential-error} is the heart of our analysis and has several simple proofs. We present these varied views as they provide many useful algebraic, geometric, and algorithmic insights into least-squares classification and interpolation.

\paragraph{Algebraic proof.} The algebraic proof of Lemma~\ref{lemma:sequential-error} follows from formulae for matrix inverses. Let $K$ denote the kernel matrix for $S_{i}$. Partition $K$ as
$$
	K = \begin{bmatrix} K_{11} & K_{12} \\ K_{21} & K_{22} \end{bmatrix}
$$
where $K_{11}$ is $(i-1) \times (i-1)$  and $K_{22}$ is a scalar equal to $K(x_i,x_i)$. Note that under this partitioning,
$$
\widehat{f}_{S_{i-1}}(x_i) = K_{21} K_{11}^{-1} y_{1:i-1}\, ,
$$
where $y_{1:i-1}$ slices all but the last element of $y$.

First, note that 
$$
s_i^2 = K_{22}-K_{21} K_{11}^{-1} K_{12}\,.
$$
Next, using the formula for inverting partitioned matrices, we have that 
$$
	K^{-1} = \begin{bmatrix} 
	(K_{11}-K_{12}K_{22}^{-1} K_{21})^{-1} & s_i^{-2} K_{11}^{-1} K_{12} \\
	s_i^{-2} (K_{11}^{-1} K_{12})^\top & s_i^{-2}
	\end{bmatrix}\,.
$$
By the Woodbury formula we have
$$
	(K_{11}-K_{12}K_{22}^{-1} K_{21})^{-1}
	= K_{11}^{-1} +s_i^{-2} \left(K_{21} K_{11}^{-1}\right)^\top \left(K_{21} K_{11}^{-1}\right) \,.
$$
Hence,
$$
	\|\widehat{f}_{S_i}\|_K^2 = y^\top K^{-1} y = s_i^{-2}(y_i - 2y_i \widehat{f}_{S_{i-1}}(x_i) + \widehat{f}_{S_{i-1}}^2(x_i) ) + y_{1:i-1}^\top K_{11}^{-1} y_{1:i-1}\,.
$$
Rearranging terms  proves the lemma.

\paragraph{Geometric Proof.} We can sketch a geometric proof Lemma~\ref{lemma:sequential-error} which gives more light into the sequential nature of our regret bound. This proof naturally gives rise to an online algorithm for solving MNIC, demonstrating that at each step we only need to increment our function in a direction orthogonal to all of the previously seen examples.

\begin{figure}[h]
	\centering
	\includegraphics[width=0.7\linewidth]{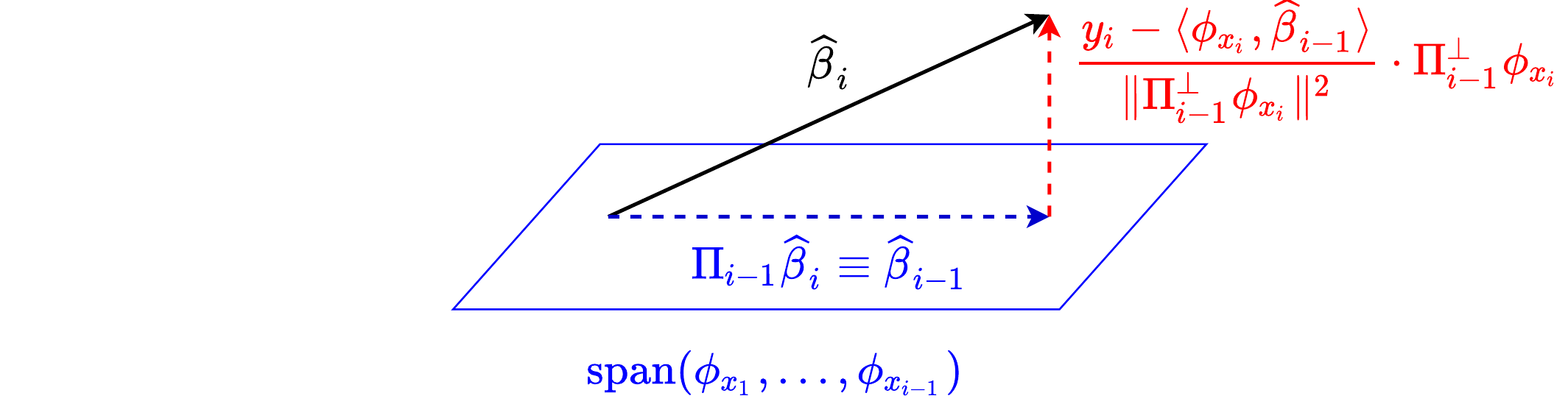}
	\caption{Illustration of the geometric proof.}
\end{figure}
	
Recall that $\widehat{f}_{S_i}(x)$ is a linear function in $\phi_x$
\begin{align}
	\widehat{f}_{S_i}(x) = \langle \phi_x, \widehat{\beta}_{i} \rangle_{K} \enspace,
\end{align}
where $\widehat{\beta}_{i}$ lies in the $\spn(\phi_{x_1}, \ldots, \phi_{x_{i}})$. Let us compare $\widehat \beta_i$ and $\widehat \beta_{i-1}$. Let $\Pi_{i-1}$ denote the orthogonal projection onto $\spn(\phi_{x_1}, \ldots, \phi_{x_{i-1}})$. Observe that one must have
\begin{align}
	\Pi_{i-1} \widehat\beta_{i} = \widehat\beta_{i-1} \enspace.
\end{align}
This is because the function $g_1(x) := \langle \phi_x, \Pi_{i-1} \widehat\beta_{i}\rangle_K$ has to interpolate the data set $S_{i-1}$ as $$g_1(x_j) \equiv \widehat{f}_{S_i}(x_j)$$ for $1\leq j \leq i-1$. On the span of $\Pi_{i-1}$, there is a unique interpolating solution to $S_{i-1}$, namely $\widehat\beta_{i-1}$, thus proving the observation. Therefore, since $\widehat \beta_i$ lies in the span of $\{\phi_{x_1}, \ldots, \phi_{x_{i-1}} \} \cup \{ \phi_{x_{i}} \}$, 
\begin{align}
	\widehat \beta_i &= \widehat \beta_{i-1} + \Pi_{i-1}^{\perp} \widehat \beta_i  \\
		&=  \widehat \beta_{i-1} + c \cdot \Pi_{i-1}^{\perp} \phi_{x_i}
\end{align}
with some constant $c$. To compute the value of $c$, let us apply the function to the data point $(x_i, y_i)$
\begin{align}
		y_i = \widehat{f}_{S_i}(x_i) = \langle \phi_{x_i}, \widehat{\beta}_{i} \rangle_{K} &= \langle \phi_{x_i}, \widehat{\beta}_{i-1} \rangle_{K} + c \cdot \langle \phi_{x_i},  \Pi_{i-1}^{\perp} \phi_{x_i} \rangle_{K}\\
		&=\widehat{f}_{S_{i-1}}(x_i) + c \cdot s_i^2 \enspace.
\end{align}
Thus, we have proven
\begin{align}
		\widehat \beta_i = \widehat \beta_{i-1} + \frac{y_i - \widehat{f}_{S_{i-1}}(x_i)}{s_i^2} \cdot \Pi_{i-1}^{\perp} \phi_{x_i} \enspace.
	\end{align}
Evaluating the norm on both sides of the equation will conclude the proof.

\paragraph{An Online Interpolation Algorithm.}

The geometric proof of Lemma~\ref{lemma:sequential-error} describes the following online algorithm for MNIC:

\iftoggle{arxiv}{\begin{algorithm}}{\begin{algorithm2e}}
 \KwData{A sequence of data $\{z_i\}_{i=1}^{n}$. A feature embedding $\phi:\Reals^d \rightarrow \Reals^p$ that corresponds to an RKHS.}
 \KwResult{A min-norm interpolation function $\widehat{f}_{S_{i-1}}: \cX \rightarrow \cY$ on all past data $S_{i-1}$ at each time stamp $i$.}
 Initialization: Set $\widehat \beta_0 = \mathbf{0} \in \Reals^p$ and $\widehat f_{S_0} = 0$, and $i=0$ \;
 \While{$i < n$}{
   Set $i = i+1$ and receive a new data $z_i = (x_i, y_i)$ \;
   Make prediction based on the previous interpolator $\widehat f_{S_{i-1}}$ and record the error $\epsilon_i = y_i - f_{S_{i-1}}(x_i)$. 
   Update the vector $\widehat \beta_i$
   $$
   \widehat \beta_i = \widehat \beta_{i-1} + \frac{\epsilon_i}{\| \Pi_{i-1}^{\perp} \phi_{x_i} \|^2} \cdot \Pi_{i-1}^{\perp} \phi_{x_i} \enspace,
   $$
   and the corresponding min-norm interpolation function $\widehat{f}_{S_{i}}(x) = \langle \phi_x,  \widehat \beta_i \rangle$ \;
 }
 \caption{Online Minimum-Norm Interpolation Algorithm.}\label{alg:online-alg}
\iftoggle{arxiv}{\end{algorithm}}{\end{algorithm2e}}
\medskip

Note that the function update can be written equivalently in the kernel form: define $g_i(x)$
\begin{align}
	g_i(x) = K(x, x_i) -  K(x, X_{i-1}) [K(X_{i-1}, X_{i-1}) ]^{-1} K(X_{i-1}, x_i) 
\end{align}
with $X_{i-1}$ concatenating $\{x_1, \ldots, x_{i-1}\}$, then
\begin{align}
	\widehat{f}_{S_{i}}(x) = \widehat{f}_{S_{i-1}}(x) + \frac{\epsilon_i}{g_i(x_i)}  g_i(x) \enspace.
\end{align}

Theorem~\ref{thm:regret} can be interpreted as deriving a regret bound for this online algorithm.

\paragraph{Connections to the QR decomposition.} The online algorithm defined by Algorithm~\ref{alg:online-alg} turns out to be equivalent to solving MNIC with a QR factorization. Recall that the QR factorization of a matrix $A$ writes $A=QR$ where $Q$ has orthonormal columns and $R$ is upper triangular. If $A$ is a data matrix with each column a data point, then $Q$ is an orthonormal basis for the span of the data, and hence we can compute the projection matrices $\Pi$ directly from $Q$. We formalize this observation in the following Proposition.

\begin{proposition}\label{prop:qr}
	Suppose $\phi$ maps into a $p$ dimensional space and $K(x, x') = \langle \phi_x, \phi_{x'}\rangle$ with the inner product in $\Reals^p$. Define the $p\times n$ dimensional matrix $\Phi$ to have $i$th column equal to $\phi_{x_i}$. Let $\Phi = QR$ be a QR decomposition of $\Phi$ and denote the $i$th column of $Q$ by $q_i$. Let $z=R^{-\top} y$. Then $\widehat \beta_k = \sum_{i=1}^k q_{i} z_i$.
\end{proposition}

We defer the proof of this proposition to Appendix~\ref{sec:qr-proof}. Proposition~\ref{prop:qr} also allows us to compute the average of $\widehat{\beta}_i$ from the QR decomposition. Let $D$ denote the diagonal matrix with $k$th diagonal entry $1-\tfrac{k-1}{n}$. Then Proposition~\ref{prop:qr}{ immediately implies that 
\begin{equation}
	\frac{1}{n} \sum_{i=1}^n \widehat{\beta}_i = Q D R^{-\top} y\,.
\end{equation}
This implies that computing the average of $\widehat\beta$ merely requires only $n$ floating point operations more than solving for the least-norm solution of $\Phi^\top \beta =y$.

We can also kernelize the computation of the average. With slight abuse of notation, we can still let $\Phi$ denote the semi-infinite dimensional data matrix consisting of concatenation of $\phi_{x_i}$. This $\Phi$ will have a QR decomposition.  The kernel matrix $K = \Phi^\top \Phi$ will have Cholesky decomposition $R^\top R$ where $R$ is the same matrix as in the QR decomposition of $\Phi$. Our goal is to write $\widehat\beta_k$ as $\sum_{i=1}^n c_i \phi_{x_i}$. Let $P_k$ denote the diagonal matrix that projects onto the first $k$ standard coordinates in $\R^n$. Then we have $\widehat\beta_k = Q P_k z$ and
$$
	Kc = \Phi^\top Q P_k z = R^\top P_k z\,.
$$
Hence, $\widehat\beta_k = \Phi c$ with $c=R^{-1} P_k R^{-\top} y$ and $\frac{1}{n} \sum_{i=1}^n \widehat\beta_k = \Phi \bar{c}$ with $\bar{c} = R^{-1} D R^{-\top} y.$

\section{Generalization to Ridge Regression}
\label{sec:ridge}

By replacing the matrix $K$ with $K+\lambda I$ in Lemma~\ref{lemma:sequential-error}, all of the results derived so far immediately generalize to kernel ridge regression and RLSC. Recall that ridge regression solves the optimization problem
$$
	\sum_{i=1}^n (f(x_i) - y_i)^2 + \lambda \|f\|^2_K\,.
$$
The optimal solution of the problem is
$$
	\widehat{f}_{S_n,\lambda} = \sum_{i=1}^n \alpha_i K(x_i, x)\,,~\text{where}~\alpha = (K+\lambda I)^{-1} y\,.
$$
Thus, the only thing distinguishing the solution of the minimum norm interpolation problem~\eqref{eq:rkhs-interpolate} is the addition of $\lambda I$. Hence, Lemma~\ref{lemma:sequential-error} holds, with slightly different complexity measures. We need to replace the norm of $\widehat{f}_{S_n,\lambda}$ with the complexity measure
$$
	B_{n, \lambda}^2 :=y^\top (K+\lambda I)^{-1} y\,.
$$
This complexity measure upper bounds the  norm of the ridge regression solution. To see this, observe
$
y^T(K+\lambda I)^{-1} y \geq y^T(K+\lambda I)^{-1} K (K+\lambda I)^{-1} y = \alpha^T K \alpha = \|\widehat {f}_{S_n,\lambda}\|_K^2 \,.
$
We additionally need to replace $s_i^2$ with
$
	s^2_{i, \lambda} := \min_{c \in \Reals^{i-1}} \,\, \big\| \phi_{x_i} - \sum_{j=1}^{i-1} c_j  \phi_{x_j} \big\|_{K}^2 + \lambda \|c\|_2^2 + \lambda\,.
$
With these substitutions, analogs of Theorems~\ref{thm:regret} and Theorem~\ref{thm:mnic-generalization} follow without modification of the proofs.

\section{Some Plausible Scenarios for Slow Norm Growth}
\label{sec:norm-growth}

We now turn to understanding the magnitude of the norm $\| \widehat{f}_{S_n} \|_K^2$ under different generative models. Provided there is some separation between the probability distributions of the two classes, we expect the norm to grow slowly as a function of the sample size. As we shall see, we need not assume that the data realizations are well separated or that there is no noise on the labels. Rather, we will only need that $p(x|y=1)$ and $p(x|y=-1)$ are sufficiently distinct at the population level.

We begin with a simple over-parametrized Gaussian mixture model in Section~\ref{sec:linear-kernel}. In this case, the interpolant's norm stays bounded almost surely, implying an $O(1)$ regret with probability one. This example is further extended to more general mixture models in Section~\ref{sec:general-mixture}. This concerns a much broader range of scenarios where a $o(n)$ regret (sublinear) is possible. In Section~\ref{sec:nonlinear-kernel}, we study general non-linear kernels with considerably weaker distributional assumptions. There, we relate the magnitude of the interpolant's norm to the Total Variation distance between the class conditional distributions. We defer all proofs to the Appendix.

\subsection{Over-parametrized Gaussian Mixture Model}
\label{sec:linear-kernel}
	The first case we consider is a linear kernel with $\phi_x = x \in \Reals^d$  and $K(x, x') = \langle x, x' \rangle$. Assume 
	\begin{align}\label{eq:gen1}
		x_i =  y_i  \cdot \mu  \cdot \theta_\star + \epsilon_i 
	\end{align}
where $\mu$ parametrizes the strength of the signal, $\theta_\star$ the direction with $\| \theta_\star \|_2 = 1$, and $\epsilon_i \sim \cN(0, 1/d \cdot I_d)$ the noise (independent of $y_i$). We consider the over-parametrized regime with $d(n)/n = \psi > 1$.

\begin{lemma}
	\label{lem:GMM_linear}
	For i.i.d. data generated from the model defined by~\eqref{eq:gen1}, $R_n^2 \leq (\mu + 1)^2$ a.s. and
	\begin{align}
		\limsup_{n\rightarrow \infty} ~~\| \widehat{f}_{S_n} \|_K^2 \leq \frac{\psi}{(\psi-1)\mu^2},~~ \text{a.s.}
	\end{align}
\end{lemma}

This Lemma and Theorem~\ref{thm:regret} immediately imply that for the Gaussian mixture model~\eqref{eq:gen1}, 
	\begin{align}
		 \limsup_{n\rightarrow \infty} \sum_{i=1}^n \mathbbm{1}\big\{ y_i \widehat{f}_{S_{i-1}} (x_i) \leq 0 \big\} \leq (\mu+1)^2 \frac{\psi}{(\psi-1)\mu^2}   \enspace, ~~\text{a.s.}
	\end{align}
In other words, we make at most a constant number of errors with probability one in the infinite sequence of problems.

We can also establish a lower bound on $r_n^2$, which appears in the lower bound of Theorem~\ref{thm:regret}. This model shows that our upper bound in Theorem~\ref{thm:regret} cannot be significantly improved.
\begin{lemma}
	\label{lem:K-upp-low}
	The following bound holds under the same setting as in Lemma~\ref{lem:GMM_linear}
	\begin{align}
		r_n^2 &\geq (\mu^2+1) \frac{1}{1+C_{\mu, \psi} \cdot n},~~ \text{a.s.}
	\end{align}
	where $C_{\mu, \psi}>0$ is a constant that does not depend on $n$.
\end{lemma}

\subsection{Generalizations of the Mixture Model}
\label{sec:general-mixture}

We now turn to a more general mixture model where we allow for more general covariance structures. In doing so, we will find a variety of data generating processes that exhibit slow norm growth, but also show that there may be cases where this is not the case. 
Consider
$\phi_{x_i} \in \Reals^p$ and $K(x, x') = \langle \phi_{x}, \phi_{x'} \rangle_K$ with the inner product in $\Reals^p$. Suppose the following generative structure holds with a positive-definite covariance matrix $\Sigma \in \Reals^{p \times p}$ 
\begin{align}
	 \phi_{x_i} = y_i \cdot \mu \cdot \theta_\star + \Sigma^{1/2} \epsilon_i\enspace,
\end{align}
where $\mu > 0$ is the signal strength and  $\theta_\star$ is a unit vector in $\Reals^p$. Let $\lambda_1(\Sigma), \ldots, \lambda_p(\Sigma) >0$ denote the non-increasing sequence of eigenvalues of $\Sigma$. 
\begin{assumption}[Weak Moment Conditions]
	\label{asmp:moments}
	$\epsilon_{ij}, 1\leq i\leq n, 1\leq j \leq p(n)$ are i.i.d. entries from a distribution with, zero first-moment, unit second-moment and uniformly-bounded $m$-th moment, with $m\geq 8$. 
\end{assumption}
The zero first-moment and unit second-moment conditions can be assumed without loss of generality. The more significant assumptions made here are the requirement of bounded $m$th-moments and the i.i.d. entries in $\epsilon_i$. We leave as future work relaxing this assumption with a small-ball analysis or some other more sophisticated random matrix theory.

\begin{assumption}[Over-parametrization]
	\label{asmp:overparam}
	For sufficiently large $C>0$, we have
	 $ \frac{p}{\log p} > C \cdot n \;.$
\end{assumption}
This assumption ensures that the feature map is sufficiently over-parametrized such that the empirical kernel matrix is of full rank $n$.

With the above, the following non-asymptotic bound holds
\begin{theorem}
	\label{thm:general-cov}
	Under the Assumptions~\ref{asmp:moments} and \ref{asmp:overparam},
	the following bound holds almost surely,
	\begin{align}
		\frac{1}{n} \sum_{i=1}^n \big( y_i - \widehat f_{S_{i-1}} (x_i) \big)^2 \leq R_n^2 \frac{\| \widehat f_{S_n} \|_K^2 }{n} \leq c_1 \cdot \big( \mu^2 + {\rm tr}(\Sigma) + \gamma_p \big)  \frac{ \frac{1}{\lambda_p(\Sigma)} \frac{n}{p} \left(  1 +  c_2 \cdot \frac{ \theta_\star^\top \Sigma \theta_\star }{\lambda_p(\Sigma)} \frac{n}{p}  \right)}{ 1 + c_3 \cdot \frac{\big( \mu^2 + \theta_\star^\top \Sigma \theta_\star  \big)}{\lambda_1(\Sigma)}  \frac{n}{p} } \cdot  \frac{1}{n}
	\end{align}
	where $\gamma_p = p^{\frac{1}{2} + \frac{2}{m}} (\log p)^{0.51}$ and $c_j,1\leq j\leq 3$ are universal constants.
\end{theorem}

\begin{corollary}
	\label{coro:general-ridge}
	Under the same assumptions and notations as in Theorem~\ref{thm:general-cov}, now consider the ridge regularized predictor $\widehat{f}_{S_i, \lambda_\star}$ in Section~\ref{sec:ridge}, with an explicit regularization $\lambda_\star \geq 0$. The following bound holds almost surely,
	\begin{align}
		&\frac{1}{n} \sum_{i=1}^n \big( y_i - \widehat f_{S_{i-1}, \lambda_\star} (x_i) \big)^2 \leq R^2_{n, \lambda_\star} \frac{y^\top \big[  K + \lambda_\star I  \big]^{-1} y}{n} \nonumber\\
		& \quad \quad \quad \quad \leq c_1 \cdot \big( \lambda_\star + \mu^2 + {\rm tr}(\Sigma) + \gamma_p \big)  \frac{ \frac{ 1}{\lambda_\star/p + \lambda_p(\Sigma) } \frac{n}{p} \left(  1 +  c_2 \cdot \frac{ \theta_\star^\top \Sigma \theta_\star }{\lambda_\star/p + \lambda_p(\Sigma) } \frac{n}{p}  \right)}{ 1 + c_3 \cdot \frac{ (\mu^2 + \theta_\star^\top \Sigma \theta_\star ) }{ \lambda_\star/p +  \lambda_1(\Sigma) } \frac{n}{p}  } \cdot  \frac{1}{n} \enspace.
	\end{align}
\end{corollary}

\begin{figure}[ht]
	\centering
	\includegraphics[width=0.5\textwidth]{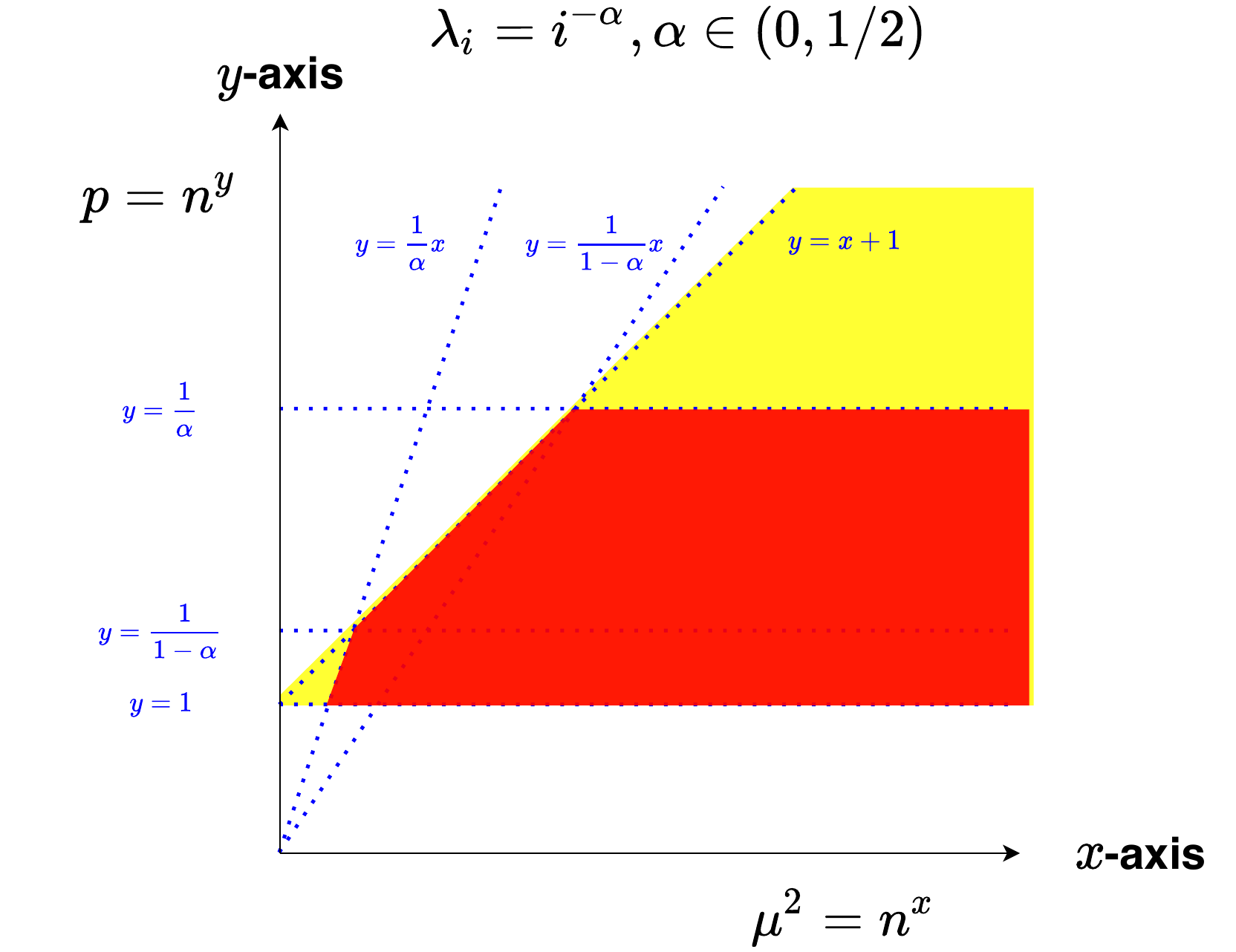}
	\caption{Illustration diagram on when slow norm growth is possible. $x$-axis corresponds to $\log(\mu^2)/\log(n)$, and $y$-axis corresponds to $\log (p)/\log (n)$. Red region shows when the upper bound in Theorem~\ref{thm:general-cov} is $o(1)$, the extended Yellow region demonstrates that with a properly chosen $\lambda_\star$, the upper bound in Corollary~\ref{coro:general-ridge} is $o(1)$.}
	\label{fig:snr}
\end{figure}

Let's unpack the bound in Theorem~\ref{thm:general-cov} by exploring some examples, highlighting the impact of the signal strength $\mu$, over-parametrization $p$, and error covariance $\Sigma$ on the regret. Consider a range of problem instances with sub-Gaussian $\epsilon_{ij}$ where 
\begin{enumerate}
 \item  $\lambda_i(\Sigma) = i^{-\alpha}, 1\leq i\leq p$ with some $\alpha \in (0, 1/2)$,
 \item $\mu^2 = n^{x}$ for $x \in (0, \infty)$, 
 \item $p = n^{y}$ for $y \in (1, \infty)$.
\end{enumerate}
For any fixed $\alpha \in (0, 1/2)$, Figure~\ref{fig:snr} illustrates the region on the $x, y$ domain when the upper bounds in Theorem~\ref{thm:general-cov} in Corollary~\ref{coro:general-ridge} are $o(1)$, asserting the plausible scenarios for slow norm growth. The detailed calculations are deferred to Appendix~\ref{sec:auxilary}. Conceptually, when the decay of eigenvalues is fast, the empirical kernel matrix will become ill-conditioned, suggesting that the interpolant's norm could grow rapidly. The region where such ill-conditioning is ruled out is colored red in Figure~\ref{fig:snr}. Explicit regularization can mitigate the rapid norm growth, as shown in the yellow region.

\section{Separation and Slow Norm Growth}
\label{sec:nonlinear-kernel}

We close by considering binary classification with general non-linear kernels. Let $\eta_\star$ denote the conditional expectation of $y$ given $x$
	\begin{align}
		\eta_\star(x) = \mathbb{P}(\by = +1 | \bx = x) - \mathbb{P}(\by = -1 | \bx = x) \enspace.
	\end{align}
	We assume that with some $B_\star>0$
	$$
	\| \eta_\star \|_{K}^2 \leq B_\star^2  \enspace.
	$$
	
	\begin{lemma}
		\label{lem:non-linear}
		Consider $S_n$ consists of i.i.d. data drawn from some distribution $\cD$.
		Denote $X$ as the concatenated data matrix of $\{x_i\}_{i=1}^n$. Conditional on the design $X$, the following bound holds,
		\begin{align}
			\E\big[ \| \widehat{f}_{S_n} \|_K^2 ~|~ X \big] \leq \| \eta_\star \|_{K}^2 +  \frac{1}{r_n^2} \sum_{i=1}^n \big(1-  \eta_\star^2(x_i) \big)  \enspace.
		\end{align}
	\end{lemma}
	
Lemma~\ref{lem:non-linear} and Theorem~\ref{thm:regret} together imply
\begin{align}\label{eq:bayes-bound}
   \frac{1}{n} \sum_{i=1}^n \mathbb{P}\big[ y_i \widehat{f}_{S_{i-1}} (x_i) \leq 0 ~|~ X \big] 
  \leq  \frac{R_n^2}{r_n^2} \cdot \frac{1}{n} \sum_{i=1}^n ( 1- \eta_\star^2(x_i))+ \frac{R_n^2 B_\star^2}{n} \enspace.
\end{align}
The second term here is a generalization error that tends to zero with $n$. The first term is proportional to the \emph{Bayes error}
$$
	\mathcal{E}(X)= \frac{1}{n} \sum_{i=1}^n ( 1- \eta_\star^2(x_i) )
$$
Note that $\mathcal{E}(X)$ is nonnegative and measures a form of accuracy of the regression function on the sample.  $\mathcal{E}(X)$  is equal to zero when the regression function makes correct assignments for all data points and equal to $1$ if all of the points are impossible to distinguish under the generative model.  In order for the bound~\eqref{eq:bayes-bound} to imply low pediction error, we need the Bayes error to be small.

Two natural conditions under which the Bayes error can be further controlled are the Mammen-Tsybakov and Massart noise conditions. A slightly stronger version of Mammen-Tsybakov condition asserts that there exists some $\alpha \in [0,1]$ and $C_0>0$ such that for all $ t\in [0,1]$
\begin{align}
			\mathbb{P}\big[ |\eta_\star(\bx)| \leq t \big]
\leq C_0 \cdot t^{\frac{\alpha}{1-\alpha}}	\enspace.
	\end{align}
Under this noise condition, we have the following upper bound on the expected Bayes Error
		\begin{align}
			&\E \big[ \mathcal{E}(X) \big] = 1 - \E[ |\eta_\star( \bx)|^2]  = 1 - \int_{0}^{1} \mathbb{P}( |\eta_\star(\bx)|>\sqrt{t} ) \dif t \\
			& \leq 1- \int_{0}^{1} \big( 1-   C_0 \cdot t^{\frac{\alpha}{2(1-\alpha)}} \big) \dif t \leq C_0 \cdot \frac{2(1-\alpha)}{2-\alpha} \enspace.
		\end{align}
Note that when $\alpha = 1$ (i.e., under the Massart noise condition), the expected Bayes Error is 0.

We can construct more general conditions under which we expect $\mathcal{E}(X)$ to be small. Let $P_{+}$ denote the conditional distribution of $x$ given $y = +1$, and correspondingly define $P_{-}$. Let $d_{\rm TV}\left( P_{+}, P_{-} \right)$ denote the total variation distance between these two distributions.

\begin{lemma}
		\label{lem:separation} For any $\epsilon>0$,
\begin{align}
	\Pr\left[\mathcal{E}(X) \leq 4\epsilon \right] \geq  1 -\tfrac{\epsilon^{-1}+1}{2} n (1- d_{\rm TV}\left( P_{+}, P_{-} \right))  \enspace.
\end{align}
\end{lemma}

Lemma~\ref{lem:separation} gives a general way to bound the expected Bayes error by lower bounding the total variation between the class conditional distributions. For example, in the Gaussian mixture model in Section~\ref{sec:linear-kernel}, we can choose  $\epsilon = n^{-1}$ and have with probability at least
$
	1 - n^2\exp(-d\mu^2/2) = 1- o(1)
$
that the error is $O(1/n)$.

\section{Conclusions and Future Work}

The linear algebraic structure of MNIC and RLSC led to surprisingly simple generalization bounds. Moreover, this structure led to tractable analysis of the norms of the solutions of a variety of classification problems. As we mention in Section~\ref{sec:mistake}, our regret and generalization bounds apply to regression problems, but we do not know which data generating models will have solutions whose norm grows slowly with the number of data points. For instance, the standard planted model $y_i=x_i^\top \beta+\epsilon_i$, where $x_i$ is isotropic Gaussian and $\epsilon_i$ is independent noise, results in an interpolating solution whose norm scales linearly with $n$. It would be fascinating to understand plausible models for regression where this norm grows sublinearly.

Future work should also investigate whether other ERM problems admit simple mistake bounds as we now see for the perceptron, the SVM, MNIC, and RLSC. Is there a general proof that connects these aforementioned algorithms? Is there a general theory yielding simple $1/n$ rates for models learned by Stochastic Gradient Descent when there is an empirical risk minimizer has zero loss?


\iftoggle{arxiv}
{{\small
\bibliographystyle{abbrvnat}
\bibliography{interpolation}
}}
{
\bibliography{interpolation}
\newpage
}

\appendix

\section{ Proof of Proposition~\ref{prop:qr}}\label{sec:qr-proof}

First note that by construction, 
$$
	\phi_{x_k} = \sum_{i=1}^{k} R_{ik} q_i\,.
$$
Hence, $\Pi_{k-1}^{\perp} \phi_{x_k} = R_{kk} q_k$. From this same calculation, we have that $s_k = R_{kk}$.

Consider the vector $z=R^{-\top} y$. We have
$$
	z_k = \frac{y_k - \sum_{i=1}^{k-1} R_{ik} z_i}{R_{kk}}\,.
$$

Now let us proceed by induction. Certainly, for the base case, we have 
$$
\widehat{\beta}_1 = \frac{y_1}{s_1^2}  \phi_{x_i}=  \frac{y_1}{R_{11}^2} R_{11} q_1 = z_1 q_1\,.
$$ 

Now, suppose the formula holds for $i<k$ and consider 
$$
\begin{aligned}
\widehat{f}_{S_{k-1}}(x_k) = \langle \phi_{x_k} , \widehat\beta_{k-1} \rangle = \left\langle \sum_{j=1}^k  R_{jk} q_j, \sum_{\ell=1}^{k-1} q_\ell z_\ell \right\rangle
= \sum_{j=1}^{k-1} R_{jk} z_j\,.
\end{aligned}
$$
Thus, 
$$
	z_k = \frac{y_k -\widehat{f}_{S_{k-1}}(x_k) }{s_k}\,.
$$
and we have
$$
\widehat{\beta}_{k+1}-\widehat{\beta}_k = \frac{y_k -\widehat{f}_{S_{k-1}}(x_k) }{s_k^2} \Pi_{k-1}^{\perp} \phi_{x_k} = z_k q_k\,,
$$
which completes the proof.

\section{Probabilistic Tools for Analyzing Generative Models}

We will need two results from the literature. First, the following is Theorem 5.41 in \cite{vershynin2010introduction}.
\begin{proposition}
	\label{prop:singular-value}
	Let $A$ be an $p \times n$ random matrix whose rows $A_i$ are independent isotropic random vectors in $\Reals^n$. Assume that $\| A_i \|_2 \leq \sqrt{N}$ almost surely for all $i$. Then for every $t \geq 0$, one has the following bound on the singular values of $A$
	\begin{align}
		\sqrt{p} - t \sqrt{N} \leq s_{\min}(A) \leq s_{\max}(A) \leq \sqrt{p} + t \sqrt{N}
	\end{align}
	with probability at least $1-2n \cdot \exp(-ct^2)$, where $c>0$ is an absolute constant.
\end{proposition}

The next proposition follows from Lemma A.3 in \cite{karoui2010SpectrumKernel}.
\begin{proposition}
	\label{prop:quadratic-form}
	Let $\{ z_i \}_{i=1}^n$ be i.i.d. random vectors in $\Reals^p$, whose entries are i.i.d., mean $0$, variance $1$ and have bounded $m$-absolute moments ($m>4$). Suppose that $\{ \Sigma_p \}$ is a sequence of positive semi-definite matrices whose operator norms are uniformly bounded and $n/p$ is asymptotically bounded. We have, 
	\begin{align}
		\max_{1\leq i \leq n} \left| z_i^\top \Sigma_p z_i - {\rm tr}(\Sigma_p) \right| \leq p^{\frac{1}{2}+ \frac{2}{m}} (\log p)^{0.51} \quad a.s.
	\end{align} 
\end{proposition}

The following propositions are useful calculations needed for our proofs. The following proposition follows from rotational invariance of Gaussian and standard concentration of $\chi^2$ random variables.
\begin{proposition}
	\label{prop:concentration}
	Let $G \in \mathbb{R}^{n\times d}$ with each entry i.i.d. $\cN(0, 1/d)$, then for any fixed vector $v \in \Reals^n$
	\begin{align}
		v^\top [G G^\top]^{-1} v \stackrel{dist.}{=}  \| v \|^2 \cdot \frac{d}{\chi^2_{d-n+1}} \enspace.
	\end{align}
	In the case with $d/n=\psi > 1$, with probability at least $1 - 2\exp(-t)$, 
	\begin{align}
		 (1-\psi^{-1}) - 2\sqrt{\psi^{-1} (1-\psi^{-1}) \frac{t}{n}} \leq  \frac{\|  v\|^2}{ v^\top [G G^\top]^{-1} v} \leq (1-\psi^{-1}) + 2\sqrt{\psi^{-1} (1-\psi^{-1}) \frac{t}{n}} + 2 \psi^{-1} \frac{t}{n} \enspace.
	\end{align}
\end{proposition}
\begin{proof}
	Due to rotational invariance, $v^\top [G G^\top]^{-1} v \stackrel{dist.}{=}  \| v \|^2 K_{11}$ where $K_{11}$ is the top-left element of the matrix $[GG^\top]^{-1}$. It is easy to see that $K_{11}$ follows the same distribution as $d/\chi^2_{d-(n-1)}$. The proof completes using the standard concentration of $\chi^2$ random variables (see for instance,  Massart-Laurent).
\end{proof}

The following proposition controls some matrix quantities needed in our analysis.
\begin{proposition}
	\label{prop:concentration-1}
	Let $E \in \Reals^{n \times p}$ be random matrix with i.i.d. entries, zero-mean and bounded-$m$ moments. $E_{i,\cdot} \in \Reals^p$ denotes the $i$-th row of $E$, and $E_{\cdot,j} \in \mathbb{R}^n$ denotes the $j$-th column of $E$.
	The following bounds hold almost surely if $m> 4$
	\begin{align}
		&\max_{1\leq i\leq n} \left| E_{i,\cdot}^\top \Sigma_p E_{i,\cdot}  -  {\rm tr}(\Sigma_p)  \right| \leq p^{\frac{1}{2}+ \frac{2}{m}} (\log p)^{0.51}\enspace, \\
		&\max_{1\leq j\leq p} \left| E_{\cdot,j}^\top  E_{\cdot,j}  -  n  \right| \leq p^{\frac{1}{2}+ \frac{2}{m}} (\log p)^{0.51}\enspace, \\
		&\lambda_1(E \Sigma_p E^\top ) \leq \lambda_1(\Sigma_p) \cdot \big( \sqrt{p} + \sqrt{n (\log n)} + \sqrt{p^{\frac{1}{2} + \frac{2}{m}} (\log p)^{0.51} (\log n)} \big)^2 \leq 1.5 \lambda_1(\Sigma_p) \cdot p \enspace, \\
		&\lambda_n(E \Sigma_p E^\top ) \geq \lambda_p(\Sigma_p) \cdot \big( \sqrt{p} - \sqrt{n (\log n)} - \sqrt{p^{\frac{1}{2} + \frac{2}{m}} (\log p)^{0.51} (\log n)}  \big)^2 \geq 0.5  \lambda_p(\Sigma_p) \cdot p \enspace.
	\end{align}
\end{proposition}
\begin{proof}
	This is a direct implication of Proposition~\ref{prop:singular-value} and Proposition~\ref{prop:quadratic-form} with the choice $t = C (\log n)^{0.5}$ ($C>0$ large enough) and $N = n + p^{\frac{1}{2} + \frac{2}{m}} (\log p)^{0.51}$. Notice here that $p^{\frac{1}{2} + \frac{2}{m}} (\log p)^{0.51} \ll p$ for $m>4$.
\end{proof}

\begin{proposition}
	\label{prop:concentration-2}
	Consider the same setting as in Proposition~\ref{prop:concentration-1}.
	If assumed in addition that $m \geq 8$, then the following bounds hold almost surely for a fixed $v \in \Reals^p$
	\begin{align}
		\left| \sum_{i=1}^n \langle v, E_{i, \cdot} \rangle^2 - n \cdot \| v \|^2 \right| \leq n^{0.8} \cdot \| v \|^2 \enspace,\\
		\left| \sum_{i=1}^n \langle v, E_{i, \cdot} \rangle \right| \leq  n^{0.8} \cdot \| v \| \enspace.
	\end{align}
\end{proposition}
\begin{proof}[of Proposition~\ref{prop:concentration-2}]
	We start with the second statement. Define i.i.d. random variables $\bw_i := \langle v, E_{i, \cdot} \rangle$, then $\E[\bw_i] = 0$ and 
	\begin{align}
		\Pr\left( \big| \sum_{i=1}^n \bw_i \big| \geq t \right) &\leq \frac{\E \left[  \big(\sum_{i=1}^n \bw_i \big)^4 \right]}{t^4} \leq C \frac{ n \E[\bw_1^4] + n^2 (\E[\bw_1^2])^2 }{t^4} \\
		\Pr\left( \big| \sum_{i=1}^n \bw_i \big| \geq n^{0.8} \cdot \| v \| \right) & \leq C' \frac{n^2\|v \|^4}{t^4} \leq C' \frac{1}{n^{1.2}} \enspace,
	\end{align}
	with choice $t = n^{0.8}\| v \|$. Since the tail probability is summable w.r.t. $n$, Borell-Cantelli Lemma proved the almost surely statement.

	Define $\bw_i^2 = \langle v, E_{i, \cdot} \rangle^2$, then $\E [\bw_i^2] = \| v\|^2$ and 
	\begin{align}
		\E[\bw_i^4]  = \E\left[ \big( \sum_{j\in [p]} v_j E_{i,j}\big)^4 \right] \leq C \big( \sum_{j} v_j^4 + \sum_{j \neq j'} v_{j}^2 v_{j'}^2 \big) \leq C \| v \|^4 \enspace.
	\end{align}
	By Chebyshev's inequality
	\begin{align}
		\Pr\left( \big| \sum_{i=1}^n \bw_i^2 - \E[\bw_i^2] \big| \geq t \right) \leq \frac{\E \left[  \big(\sum_{i=1}^n \bw_i^2 - \E[\bw_i^2] \big)^4 \right]}{t^4} \enspace.
	\end{align}
	It is easy to see that $$\E \left[ \big( \sum_{i=1}^n \bw_i^2 - \E[\bw_i^2] \big)^4 \right] \leq C \left( n \E[ (\bw_i^2 - \E[\bw_i^2])^4 ] + n^2  \big( \var[\bw_i^2] \big)^2 \right) \leq C' n^2 \| v \|^8,$$
	as $\var[\bw_i^2] \leq \E[\bw_i^4]$ and $\E[ (\bw_i^2 - \E[\bw_i^2])^4 ] \leq C \| v \|^8$ if $E_{ij}$ has uniformly bounded $m$-th moment with $m\geq 8$.
	Now choose $t = n^{0.8} \| v\|^2$, we have
	\begin{align}
		\Pr\left( \sum_{i=1}^n (z_i - \E[z_i]) \geq t \right) \leq C' \frac{n^2 \| v \|^8}{(n^{0.8} \| v\|^2)^4} \leq C' \frac{1}{n^{1.2}}
	\end{align}
	The proof is again complete using the Borell-Cantelli Lemma.
\end{proof}

\section{Proofs for Section~\ref{sec:norm-growth}}

In this section, we use $\Phi_n \in \mathbb{R}^{n \times p}, X_n \in \mathbb{R}^{n \times d}$ to denote the embedded feature matrix and the original design matrix respectively. $Y_n \in \mathbb{R}^n$ the response vector, and $Y_n \circ X_n = [y_1 x_1, \ldots, y_n x_n]^\top \in \mathbb{R}^{n \times d}$.

\begin{proof}[of Lemma~\ref{lem:GMM_linear}]
	It can be verified that, based on the definition of Hadamard product
	\begin{align}
		\| \widehat{f}_{S_n} \|_K^2 = \mathbf{1}^\top [ (Y_n \circ X_n) (Y_n \circ X_n)^\top ]^{-1} \mathbf{1} \enspace.
	\end{align}
	Due to the rotational invariance of Gaussian, the distribution of $\| \widehat{f}_{S_n} \|_K^2$ stays unchanged under the actions of the orthogonal groups on $\Reals^d$. Therefore one can without loss of generality assume $\theta_\star =  e_1$, then
	\begin{align}
		Y_n \circ X_n = [\mu \mathbf{1} + g, Z ], ~~ g\in \Reals^{n}, Z \in \Reals^{n\times (d-1)}
	\end{align}
	where each entry of $g, Z$ is drawn i.i.d. from $\cN(0, 1/d)$.
	
	Now, by the Woodbury formula
	\begin{align}
		\label{eqn:woodbury}
		[ (Y_n \circ X_n) (Y_n \circ X_n)^\top ]^{-1} = [Z Z^\top]^{-1} - \frac{[Z Z^\top]^{-1} (\mu \mathbf{1} + g) (\mu \mathbf{1} + g)^\top  [Z Z^\top]^{-1} }{ 1 + (\mu \mathbf{1} + g)^\top [Z Z^\top]^{-1} (\mu \mathbf{1} + g) }\enspace,
	\end{align}
	thus
	\begin{align}
		\| \widehat{f}_{S_n} \|_K^2 &\stackrel{dist.}{=} \mathbf{1}^\top [Z Z^\top]^{-1} \mathbf{1} - \frac{\big( \mathbf{1}^\top [Z Z^\top]^{-1} (\mu \mathbf{1} + g)  \big)^2}{1+ (\mu \mathbf{1} + g)^\top [Z Z^\top]^{-1} (\mu \mathbf{1} + g)} \\
		&= \frac{\mathbf{1}^\top [Z Z^\top]^{-1} \mathbf{1} \cdot \big( 1+ g^\top [Z Z^\top]^{-1} g \big) - \big(  \mathbf{1}^\top [Z Z^\top]^{-1} g \big)^2}{1+ (\mu \mathbf{1} + g)^\top [Z Z^\top]^{-1} (\mu \mathbf{1} + g)} \enspace.
	\end{align}
	Using Proposition~\ref{prop:concentration}, we have that with probability at least $1-4\exp(-t)$
	\begin{align}
		& \quad (\mu \mathbf{1} + g)^\top [Z Z^\top]^{-1} (\mu \mathbf{1} + g) \\
		& \geq \big[\mu^2 n -2 \mu \sqrt{\psi^{-1}t} + \psi^{-1}(1-2\sqrt{\frac{t}{n}}) \big]  \big( \frac{\psi}{\psi-1} \frac{1}{1+ 2 \sqrt{ \frac{1}{\psi-1} \frac{t}{n} } + 2\frac{1}{\psi-1} \frac{t}{n}   }  \big) \\
		&\geq n \mu^2 \frac{\psi}{\psi-1}  (1 - O(\sqrt{t/n} )) \enspace.
	\end{align}
	Similarly, with probability at least $1 - 5\exp(-t)$,
	\begin{align}
		g^\top [Z Z^\top]^{-1} g &\leq \psi^{-1} \big[1 + 2\sqrt{\frac{t}{n}} + 2\frac{t}{n}\big] \big( \frac{\psi}{\psi-1} \frac{1}{1- 2 \sqrt{ \frac{1}{\psi-1} \frac{t}{n} }  } \big) \\
		 &\leq \frac{1}{\psi-1} ( 1 + O(\sqrt{t/n} )) \enspace,
	\end{align}
	and
	\begin{align}
		\mathbf{1}^\top [Z Z^\top]^{-1} \mathbf{1} &\leq n   \big( \frac{\psi}{\psi-1} \frac{1}{1- 2 \sqrt{ \frac{1}{\psi-1} \frac{t}{n} }  } \big) \\
		&\leq n \frac{\psi}{\psi-1} ( 1 + O(\sqrt{t/n} )) \enspace.
	\end{align}
	Putting the above three estimates together, with probability at least $1 - 9\exp(-t)$, we have
	\begin{align}
		\| \widehat{f}_{S_n} \|_K^2 \leq \frac{\psi}{(\psi-1)\mu^2} (1 + O(\sqrt{t/n} )) \enspace.
	\end{align}
	Take $t = 11 \log n$ yields 
	\begin{align}
		\limsup_{n\rightarrow \infty} ~~\| \widehat{f}_{S_n} \|_K^2 \leq \frac{\psi}{(\psi-1)\mu^2} ,~~ \text{a.s.}
	\end{align}
	by using the Borell-Cantelli Lemma.
\end{proof}

\begin{proof}[of Lemma~\ref{lem:K-upp-low}]
	For the lower bound, let's first condition on $x_i = x$ and project $X_{S_n\backslash i}$ to the normal vector $\bar{x} = x/\| x \|$, which denoted as $v = X_{S_n\backslash i} \Pi_{x} \in \Reals^n$. We know
	\begin{align}
	& \quad x^\top x - x^\top X_{S_n\backslash i}^\top [X_{S_n\backslash i} X_{S_n\backslash i}^\top ]^{-1} X_{S_n\backslash i} x\\
	 &= \| x \|^2 \left(1 - v^\top [v v^\top + X_{S_n\backslash i} \Pi_{x}^{\perp} X_{S_n\backslash i}^\top]^{-1}  v \right)\\
	& = \| x \|^2 \left( 1-  \big[v^\top [X_{S_n\backslash i} \Pi_{x}^{\perp} X_{S_n\backslash i}^\top]^{-1} v - \frac{(v^\top [X_{S_n\backslash i} \Pi_{x}^{\perp} X_{S_n\backslash i}^\top]^{-1} v)^2}{1+v^\top [X_{S_n\backslash i} \Pi_{x}^{\perp} X_{S_n\backslash i}^\top]^{-1} v} \big] \right) \\
	&= \| x \|^2 \frac{1}{1+ v^\top [X_{S_n\backslash i} \Pi_{x}^{\perp} X_{S_n\backslash i}^\top]^{-1} v} \enspace.
	\end{align}
	Now we will upper bound $v^\top [X_{S_n\backslash i} \Pi_{x}^{\perp} X_{S_n\backslash i}^\top]^{-1} v$. Define a vector $u$, which is the projection of the matrix $X_{S_n\backslash i}$ to the vector space
	\begin{align}
		\tilde\theta := \theta_\star - \langle \theta_\star, \bar{x} \rangle \bar{x} \in \Reals^d
	\end{align}
	\begin{align}
		u = X_{S_n\backslash i} \Pi_{\tilde \theta} \enspace.
	\end{align}
	By construction, $\tilde\theta \perp x$. We can verify that
	\begin{align}
		v^\top [X_{S_n\backslash i} \Pi_{x}^{\perp} X_{S_n\backslash i}^\top]^{-1} v &\stackrel{dist.}{=} v^\top [u u^\top + Z Z^\top]^{-1} v^\top \\
		& = v^\top [ZZ^\top]^{-1} v - \frac{ (v^\top [ZZ^\top]^{-1} u )^2}{1+ u^\top [ZZ^\top]^{-1} u}
	\end{align}
	where $(v, u)$ is independent of $Z$ and that $Z \in \Reals^{n \times (d-2)}$ has i.i.d. entries $Z_{ij} \sim \cN(0, 1/d)$.

	Let's analyze the distribution of each entry of $u, v \in \Reals^n$ conditioned on $x$. Define $r = \langle \bar{x}, \theta_\star\rangle$
	\begin{align}
		v_i = \mu r  y_i + \frac{1}{\sqrt{d}}  g_i^{(1)}, ~~g_i^{(1)} \sim \cN(0, 1) \enspace, \\
		u_i = \frac{1}{\sqrt{1-r^2}} (\mu y_i - r^2) + \frac{1}{\sqrt{d}}  g_i^{(2)}, ~~g_i^{(2)} \sim \cN(0, 1) \enspace,
	\end{align}
	with $g^{(1)}, g^{(2)}, Z$ mutually independent.
	Using the Proposition~\ref{prop:concentration}, we have
		\begin{align}
			(v^\top [ZZ^\top]^{-1} v) (u^\top [ZZ^\top]^{-1} u) -  (v^\top [ZZ^\top]^{-1} u )^2 \leq \frac{\mu^2 r^6}{1-r^2} n^2 \frac{\psi}{\psi-1} (1 + O(\sqrt{\frac{\log n}{n}}))
		\end{align}
	and
	\begin{align}
		u^\top [ZZ^\top]^{-1} u \geq \mu^2 r^2 n \frac{\psi}{\psi-1}  (1 - O(\sqrt{\frac{\log n}{n}}))
	\end{align}
	Combining with the concentration bound on $r^2  =  \frac{\mu^2}{1+\mu^2} (1 + O(\sqrt{\frac{\log n}{n}}))$, we finish the proof.
\end{proof}

\begin{proof}[of Theorem~\ref{thm:general-cov}]
	 The proof proceeds in a similar way as in Lemma~\ref{lem:GMM_linear},
	 $$
	 \| \widehat f_{S_n} \|_K^2 = \mathbf{1}^\top [(Y_n \circ \Phi_n) (Y_n \circ \Phi_n)^\top ]^{-1} \mathbf{1} \enspace.
	 $$
	 Define the projection matrix $\Pi_{\theta_\star}$ and $\Pi_{\theta_\star}^\perp := I_p - \Pi_{\theta_\star}$ and $E = [\epsilon_{ij}] \in \Reals^{n \times p}$. With this notation we have
	 \begin{align}
	 	(Y_n \circ \Phi_n) (Y_n \circ \Phi_n)^\top &= (\mu \mathbf{1} \theta_\star^\top + E \Sigma^{1/2} )( \Pi_{\theta_\star} + \Pi_{\theta_\star}^\perp) (\mu \mathbf{1} \theta_\star^\top + E \Sigma^{1/2} )^\top \\
		&= (\mu \mathbf{1} + g)(\mu \mathbf{1} + g)^\top + ZZ^\top
	 \end{align}
	 with $Z := E \Sigma^{1/2}\Pi_{\theta_\star}^\perp \in \Reals^{n \times p}$ and $g := E \Sigma^{1/2} \theta_\star \in \Reals^{n}$
	 \begin{align}
	 	{\rm cov}(g_i) = \theta_\star^\top \Sigma \theta_\star \\
		{\rm cov}(Z_{i}) = \Pi_{\theta_\star}^\perp \Sigma \Pi_{\theta_\star}^\perp \\
		{\rm cov}(g_i, Z_i) = \theta_\star^\top \Sigma \Pi_{\theta_\star}^\perp
	 \end{align}
	 across $1\leq i\leq n$, $\{ g_i, Z_i \}$ are mutually independent.
	 Recall Equation~\eqref{eqn:woodbury}, we know
	 \begin{align}
	 	\| \widehat f_{S_n} \|_K^2 & = \mathbf{1}^\top [Z Z^\top]^{-1} \mathbf{1} - \frac{\big( \mathbf{1}^\top [Z Z^\top]^{-1} (\mu \mathbf{1} + g)  \big)^2}{1+ (\mu \mathbf{1} + g)^\top [Z Z^\top]^{-1} (\mu \mathbf{1} + g)} \\
		&= \frac{\mathbf{1}^\top [Z Z^\top]^{-1} \mathbf{1} \cdot \big( 1+ g^\top [Z Z^\top]^{-1} g \big) - \big(  \mathbf{1}^\top [Z Z^\top]^{-1} g \big)^2}{1+ (\mu \mathbf{1} + g)^\top [Z Z^\top]^{-1} (\mu \mathbf{1} + g)} \\
		&\leq \frac{\mathbf{1}^\top [Z Z^\top]^{-1} \mathbf{1} \cdot \big( 1+ g^\top [Z Z^\top]^{-1} g \big)}{1+ (\mu \mathbf{1} + g)^\top [Z Z^\top]^{-1} (\mu \mathbf{1} + g)} \enspace.
	 \end{align}
	 Now we are going to lower bound $(\mu \mathbf{1} + g)^\top [Z Z^\top]^{-1} (\mu \mathbf{1} + g)$ and upper bound $\mathbf{1}^\top [Z Z^\top]^{-1} \mathbf{1}$ and $g^\top [Z Z^\top]^{-1} g$. We can verify that from Proposition~\ref{prop:concentration-1}
	 \begin{align}
	 	1.5 \lambda_{1}(\Sigma) p \geq \lambda_{\max}(Z Z^\top) \geq \lambda_{\min} (Z Z^\top) \geq 0.5 \lambda_{p}(\Sigma)  p
	 \end{align}
	 almost surely when $p/\log p > n$ and $m>4$. Therefore, we have 
	 \begin{align}
	 	 \mathbf{1}^\top [Z Z^\top]^{-1} \mathbf{1} &\leq \frac{\| \mathbf{1} \|^2 }{\lambda_{\min} (Z Z^\top)} \leq c'  \frac{1}{\lambda_p(\Sigma)} \frac{n}{p} \enspace, \\
		  g^\top [Z Z^\top]^{-1} g &\leq \frac{\| g \|^2 }{\lambda_{\min} (Z Z^\top)} \leq c'  \frac{ \theta_\star^\top \Sigma \theta_\star }{\lambda_p(\Sigma)} \frac{n}{p}  \enspace,\\
		  (\mu \mathbf{1} + g)^\top [Z Z^\top]^{-1} (\mu \mathbf{1} + g) &\geq \frac{ \| \mu \mathbf{1} + g \|^2 }{\lambda_{\max}(Z Z^\top)}  = \frac{ \mu^2 n + \| g\|^2 + 2\mu \langle \mathbf{1}, g \rangle }{\lambda_{\max}(Z Z^\top)} \\
		  &\geq c'' \frac{\mu^2 + \theta_\star^\top \Sigma \theta_\star}{\lambda_1(\Sigma)} \frac{n}{p} \enspace.
	 \end{align}
	 In the above equations, we used the fact that almost surely on $g$
	 \begin{align}
	 	\big| \| g \|^2  - n \cdot \theta_\star^\top \Sigma \theta_\star \big| \leq n^{0.8} \cdot \theta_\star^\top \Sigma \theta_\star, \\
		\big| \langle \mathbf{1}, g \rangle \big| \leq n^{0.8} \cdot \sqrt{\theta_\star^\top \Sigma \theta_\star}
	 \end{align}
	 by using Proposition~\ref{prop:concentration-2} with $g_i = \langle \Sigma^{1/2} \theta_\star, E_{i, \cdot} \rangle$.

	By Proposition~\ref{prop:concentration-1}, we have
	 \begin{align}
	 	R_n^2 \leq  \max_i \| y_i \phi_{x_i} \|^2 \leq 2  \big( \mu^2 + {\rm tr}(\Sigma) +  p^{\frac{1}{2} + \frac{2}{m}} (\log p)^{0.51} \big) \enspace.
	 \end{align}
	 Putting things together, we know
	 \begin{align}
	 	 R_n^2 \frac{\| \widehat f_{S_n} \|_K^2}{n} \leq  c_1 \cdot \big( \mu^2 + {\rm tr}(\Sigma) +  p^{\frac{1}{2} + \frac{2}{m}} (\log p)^{0.51} \big)  \frac{ \frac{1}{\lambda_p(\Sigma)} \frac{n}{p} \left(  1 +  c_2 \cdot \frac{ \theta_\star^\top \Sigma \theta_\star }{\lambda_p(\Sigma)} \frac{n}{p}  \right)}{ 1+ c_3 \cdot \frac{\big( \mu^2 + \theta_\star^\top \Sigma \theta_\star  \big) }{\lambda_1(\Sigma)} \frac{n}{p} } \cdot  \frac{1}{n} \enspace.
	 \end{align}
\end{proof}

\begin{proof}[of Corollary~\ref{coro:general-ridge}]
	The proof follows exactly the same steps as in Theorem~\ref{thm:general-cov}, with $\lambda_j(\Sigma)$ substituted by $\lambda_j(\Sigma) + \lambda_\star/p$ (with $j$ being either $1$ or $p$), and noting that $R_{n, \lambda_\star}^2 \leq \lambda_\star + R_n^2$. Here we also used the fact that
	\begin{align}
		Y_n^\top \big[\lambda_\star I_n + \Phi_n \Phi_n^\top \big]^{-1} Y_n = \mathbf{1}^\top \big[ \lambda_\star I_n \circ (Y_n Y_n^\top) + (Y_n \circ \Phi_n ) (Y_n \circ \Phi_n)^\top \big]^{-1} \mathbf{1} \enspace,
	\end{align}
	and the Hadamard product of $\lambda_\star I_n \circ (Y_n Y_n^\top) = \lambda_\star I_n$.
	The rest of the proof follows as in proof of Theorem~\ref{thm:general-cov} with $ZZ^\top$ replaced by $\lambda_\star I_n + ZZ^\top$.
\end{proof}

\section{Calculations for Figure~\ref{fig:snr}}
\label{sec:auxilary}

Consider $m = \infty$, namely, the tail of $\epsilon_{ij}$ behaves sub-Gaussian.
	The diagram illustrating the following cases is in Figure~\ref{fig:snr}.
	
We use the asymptotic notations $a(n)\precsim b(n)$, $a(n)\succsim b(n)$ to denote the usual asymptotic ordering. We start with the interpolation case (no explicit regularization).
	\begin{itemize}
		\item Consider the spectral decay $\lambda_1(\Sigma) = \lambda_p(\Sigma)$ and signal strength $\mu^2 \geq {\rm tr}(\Sigma)$:
		$$
		\frac{1}{n} \sum_{i=1}^n \big( y_i - \widehat f_{S_{i-1}} (x_i) \big)^2 \precsim  \mu^2  \frac{ \frac{n}{p} }{ \mu^2 \frac{n}{p}} \cdot \frac{1}{n} \precsim \frac{1}{n} \;.
		$$
		This again confirms our $O(1)$ cumulative regret result as in Lemma~\ref{lem:GMM_linear}. 
		\item Consider the spectral decay $\lambda_i(\Sigma) = i^{-\alpha}, i\in [p]$ with $\alpha\in (0, 1/2)$ and signal strength $\mu^2 \succsim p^{1-\alpha}$ (note that ${\rm tr}(\Sigma) = p^{1-\alpha} \succsim \gamma_p =  p^{\frac{1}{2}} (\log p)^{0.51}$): 
		\begin{itemize}
			\item if $ p^{1-\alpha} \succsim n \succsim p^{\alpha},$
		then
		$$
		\frac{1}{n} \sum_{i=1}^n \big( y_i - \widehat f_{S_{i-1}} (x_i) \big)^2 \precsim  \big( \mu^2 + p^{1-\alpha} \big)  \frac{ p^{\alpha} \frac{n}{p} ( 1+ o(1))}{ \mu^2 \frac{n}{p}} \cdot \frac{1}{n} \precsim \frac{p^{\alpha}}{n} = o(1) 
		$$
		\item if $p \succsim n \succsim p^{1-\alpha},$ 
		then
		$$
		\frac{1}{n} \sum_{i=1}^n \big( y_i - \widehat f_{S_{i-1}} (x_i) \big)^2 \precsim  \big( \mu^2 + p^{1-\alpha} \big)  \frac{ \big(p^{\alpha} \frac{n}{p} \big)^2 }{ \mu^2 \frac{n}{p}} \cdot \frac{1}{n} \precsim  \frac{1}{p^{1-2\alpha}} = o(1)
		$$
		\end{itemize} 
		\item Consider the spectral decay $\lambda_i(\Sigma) = i^{-\alpha}, i\in [p]$ with $\alpha\in (0, 1/2)$ and signal strength $\max\{ p/n, p^{\alpha} \} \precsim \mu^2 \precsim p^{1-\alpha}$ (note that ${\rm tr}(\Sigma) = p^{1-\alpha} \succsim \gamma_p =  p^{\frac{1}{2}} (\log p)^{0.51}$): 
		\begin{itemize}
			\item if $n \precsim p^{1-\alpha}$, then
			$$
			\frac{1}{n} \sum_{i=1}^n \big( y_i - \widehat f_{S_{i-1}} (x_i) \big)^2 \precsim p^{1-\alpha}  \frac{ p^{\alpha} \frac{n}{p} ( 1+ o(1))}{ \mu^2 \frac{n}{p}} \cdot \frac{1}{n}  \precsim \frac{p/n}{\mu^2} = o(1)
			$$
			\item if $n \succsim p^{1-\alpha}$, then
			$$
			\frac{1}{n} \sum_{i=1}^n \big( y_i - \widehat f_{S_{i-1}} (x_i) \big)^2 \precsim   p^{1-\alpha}  \frac{ \big(p^{\alpha} \frac{n}{p} \big)^2 }{ \mu^2 \frac{n}{p}} \cdot \frac{1}{n} \precsim  \frac{p^\alpha}{\mu^2} = o(1)
			$$
		\end{itemize}
	\end{itemize}
	All the above scenarios correspond to the Red region in Figure~\ref{fig:snr}.
	
	For the ridge case with explicit regularization $\lambda_\star$. First notice that the Red region certainly correspond to $o(1)$ error with the choice $\lambda_\star = 0$. Below we only consider how to extend the region with a proper choice $\lambda_\star \neq 0$.
	\begin{itemize}
		\item Consider the spectral decay $\lambda_i(\Sigma) = i^{-\alpha}, i\in [p]$ with $\alpha\in (0, 1/2)$ and signal strength $\mu^2 \succsim p/n$ and overparametrization $p^{\alpha} \succsim n$ (the upper Yellow region in Figure~\ref{fig:snr}): 
		\begin{itemize}
			\item if $\mu^2 \precsim p$, then with the choice of $\max\{ n, \mu^2, p^{1-\alpha} \} \precsim \lambda_\star \precsim  p \lambda_1(\Sigma)$
			$$
			\frac{1}{n} \sum_{i=1}^n \big( y_i - \widehat f_{S_{i-1}} (x_i) \big)^2 \precsim \lambda_\star  \frac{ \frac{n}{\lambda_\star} ( 1+ o(1))}{ 1+ \mu^2 \frac{n}{p \lambda_1(\Sigma)}} \cdot \frac{1}{n} \precsim \frac{p/n}{\mu^2} = o(1)
			$$
			\item if $\mu^2 \succsim p$, then with the choice of $  \max\{ n, \mu^2, p \lambda_1(\Sigma) \} \precsim \lambda_\star \precsim n \mu^2$
			$$
			\frac{1}{n} \sum_{i=1}^n \big( y_i - \widehat f_{S_{i-1}} (x_i) \big)^2 \precsim \lambda_\star  \frac{ \frac{n}{\lambda_\star} ( 1+ o(1))}{ 1+ \mu^2 \frac{n}{\lambda_\star}} \cdot \frac{1}{n} \precsim \frac{\lambda_\star}{n\mu^2} = o(1)
			$$
		\end{itemize}
		\item Consider the spectral decay $\lambda_i(\Sigma) = i^{-\alpha}, i\in [p]$ with $\alpha\in (0, 1/2)$ and signal strength $\mu^2 \succsim p/n$ and overparametrization $p^{\alpha} \succsim \mu^2$ (the lower Yellow region in Figure~\ref{fig:snr}): with the choice of $p/\mu^2 \precsim \lambda_\star \precsim n$, we know that $\lambda_\star \succsim p/\mu^2 \succsim p^{1-\alpha}$ and $\lambda_\star \succsim p/\mu^2 \succsim \mu^2$ (due to $p^{\alpha} \succsim \mu^2$ and $\alpha<1/2$), and therefore
		 $$
		 \frac{1}{n} \sum_{i=1}^n \big( y_i - \widehat f_{S_{i-1}} (x_i) \big)^2 \precsim \lambda_\star  \frac{ \big( \frac{n}{\lambda_\star} \big)^2 }{ 1+ \mu^2 \frac{n}{p}} \cdot \frac{1}{n} \precsim \frac{p}{\mu^2 \lambda_\star} = o(1)
	     $$
	\end{itemize}

\section{Proofs for Section~\ref{sec:nonlinear-kernel}}

	\begin{proof}[of Lemma~\ref{lem:non-linear}]
		\begin{align}
			& \quad \E\big[ \| \widehat{f}_{S_n} \|_K^2 ~|~ X_n \big] \\
			&= \E \big[ \langle Y_n Y_n^\top, [K(X_n, X_n)]^{-1} \rangle  ~|~ X_n\big] \\
			&= \langle \eta_\star(X_n)\eta_\star(X_n)^\top, [K(X_n, X_n)]^{-1} \rangle + \langle {\rm diag}\{ 1- \eta_\star^2(x_1), \ldots, 1- \eta_\star^2(x_n) \}, [K(X_n, X_n)]^{-1} \rangle  \label{eqn:conditional}\\
			& \leq \| \eta_\star \|_{K}^2 + \sum_{i=1}^n \frac{1- \eta_\star^2(x_i)}{K(x_i, x_i) - K(x_i, X_{S_n \backslash i})[K(X_{S_n \backslash i}, X_{S_n \backslash i})]^{-1} K(X_{S_n \backslash i}, x_i)} \enspace.
		\end{align}
		Here the last line uses the Riesz representation theorem that $\eta_\star(x) = \langle \eta_\star, \phi_x \rangle_{K}$, and
		\begin{align}
			\label{eqn:proj}
			\langle \eta_\star(X_n)\eta_\star(X_n)^\top, [K(X_n, X_n)]^{-1} \rangle = \| \Pi_{\phi_{X_n}} \eta_\star \|_K^2 \leq \|  \eta_\star \|_K^2 \enspace.
		\end{align}
	\end{proof}

	\begin{proof}[of Lemma~\ref{lem:separation}]
		For the first claim, observe that 
		\begin{align}
			\E \big[ \frac{1}{n}\sum_{i=1}^n \big(1- \eta_\star^2(\bx_i)) \big] &= \int \big[ 1 - \big(\frac{\dif{P_+} - \dif{P_{-}} }{\dif{P_+} + \dif{P_{-}} }  \big)^2 \big] \frac{\dif{P_+} + \dif{P_{-}} }{2} \\
			&= \int \frac{2 \dif{P_+} \dif{P_{-}}}{\dif{P_+} + \dif{P_{-}}} \leq 2 \int   \dif{P_{+}} \wedge \dif{P_{-}} \enspace.
		\end{align}

		For the second claim, let's calculate the probability of each $x_i$ falls in the region
		\begin{align}
			\cS_{\epsilon} := \big\{x \in \cX ~|~ \frac{\dif{P_{+}}}{\dif{P_{-}}}(x) \in [\epsilon, \epsilon^{-1}] \big\} \enspace.
		\end{align}
		We have
		\begin{align}
			\mathbb{P} \big[ \bx \in \cS_\epsilon \big] &= \int_{x \in \cS_\epsilon} \frac{\dif{P_{+}} + \dif{P_{-}} }{2} \leq \frac{\epsilon^{-1} + 1}{2}  \int_{x \in \cS_\epsilon} \dif{P_{+}} \wedge \dif{P_{-}} \\
			&\leq \frac{\epsilon^{-1}+1}{2} \big(1 - d_{\rm TV}( P_+, P_{-}) \big) \enspace.
		\end{align}
		Therefore, with union bound, we find
		\begin{align}
			\mathbb{P} \big[ \bx_i \notin \cS_{\epsilon}, ~~\forall 1\leq i \leq n \big] &\geq \left[ 1 - \frac{\epsilon^{-1}+1}{2} \big(1 - d_{\rm TV}( P_+, P_{-}) \big) \right]^n \\
			&\geq 1 - \frac{\epsilon^{-1}+1}{2} n \big(1 - d_{\rm TV}( P_+, P_{-}) \big)  \enspace.
		\end{align}
		For $x_i \notin \cS_{\epsilon}$, one can verify
		\begin{align}
			1 - \eta_\star^2(x_i) \leq 1 - \big(\frac{1-\epsilon}{1+\epsilon}\big)^2 \leq 4\epsilon \enspace.
		\end{align}
	\end{proof}

\end{document}